\documentclass[letterpaper, conference]{IEEEtran}
\usepackage{fullpage,graphicx,url,setspace}

\usepackage{xparse}
\usepackage{cite}
\usepackage{bm}
\usepackage{color}
\usepackage[small,bf]{caption}
\usepackage{amsmath,verbatim,amssymb,amsfonts,amscd, graphicx, algorithmic, algorithm}
\usepackage{amsmath,amsthm}
\usepackage{mathtools}
\mathtoolsset{mathic}
\usepackage{microtype}
\usepackage{wrapfig}

\renewcommand{\baselinestretch}{1}

\usepackage{ifxetex,ifluatex}
\ifxetex
\usepackage{xltxtra}
\else
\ifluatex
\usepackage{luatextra}
\else
\usepackage{amssymb}
\usepackage{algorithmic, algorithm}
\usepackage[utf8]{inputenc}

\csname else\expandafter\endcsname
\romannumeral -`0
\fi
\fi
\iftrue\relax%
\defaultfontfeatures{Ligatures=TeX}
\usepackage[math-style=TeX, vargreek-shape=TeX]{unicode-math}
\fi

\NewDocumentCommand{\entropy}{om}{\mathbb{H}\left[#2
    \IfValueT{#1}{\,\middle|\,#1}\right]}
\NewDocumentCommand{\bentropy}{lm}
  {\widetilde{\mathbb{H}}#1\left[#2\right]}

\NewDocumentCommand{\mutualInfo}{omm}{\mathbb{I}\left[#2;#3
    \IfValueT{#1}{\,\middle|\,#1}\right]}

\usepackage{tikz}
\usetikzlibrary{shapes.geometric, positioning, shadows}
\usepackage{graphicx,color}

\newtheorem{theorem}{Theorem}

\newtheorem{lemma}{Lemma}

\newtheorem{remark}{Remark}

\let\labelindent\relax
\usepackage{enumitem}
\newlist{enumerate*}{enumerate*}{1}
\setlist[enumerate*]{label=(\arabic*)}

\newcommand{\ben}{\begin{eqnarray}}

\newcommand{\een}{\end{eqnarray}}

\newcommand{\diag}{\mbox{diag}}

\title{Poisson Matrix Completion}

\author{\IEEEauthorblockN{Yang Cao}
\IEEEauthorblockA{H. Milton Stewart School of \\Industrial and Systems Engineering\\
Georgia Institute of Technology\\
caoyang@gatech.edu}
\and
\IEEEauthorblockN{Yao Xie}
\IEEEauthorblockA{H. Milton Stewart School of \\Industrial and Systems Engineering\\
Georgia Institute of Technology\\
yao.xie@isye.gatech.edu
}
}

\begin{document}
\maketitle

\begin{abstract}

We extend the theory of matrix completion to the case where we make Poisson observations for a subset of  entries of a low-rank matrix. We consider the (now) usual matrix recovery formulation through maximum likelihood with proper constraints on the matrix $M$, and establish theoretical upper and lower bounds on the recovery error. Our bounds are nearly optimal up to a factor on the order of $\mathcal{O}(\log(d_1 d_2))$. These bounds are obtained by adapting the arguments used for one-bit matrix completion \cite{davenport20121} (although these two problems are different in nature) and the adaptation requires new techniques exploiting properties of the Poisson likelihood function and tackling the difficulties posed by the locally sub-Gaussian characteristic of the Poisson distribution. Our results highlight a few important distinctions of Poisson matrix completion compared to the prior work in matrix completion including having to impose a minimum signal-to-noise requirement on each observed entry. We also develop an efficient iterative algorithm and demonstrate its good performance in recovering solar flare images.

\end{abstract}


\section{Introduction}

Matrix completion, with a goal of recovering a low-rank matrix $M \in \mathbb{R}^{d_1 \times d_2}$  from observations of a subset of its entries, attracts much interests recently due to its important real world applications including the famous Netflix problem \cite{sigkdd2007netflix}. Poisson matrix completion, where the observations are Poisson counts of a subset of the entries, is an important instance in its own as it occurs from a myriads of applications including optical imaging, nuclear medicine, low-dose x-ray imaging \cite{brady2009optical}, and network traffic analysis \cite{poissonGBG2013}.

Recently, much success has been achieved in solving the matrix completion problem using nuclear norm minimization, partly inspired by the theory of compressed sensing \cite{donoho2006compressed}. It has been shown that when $M$ is low rank, it can be recovered from only a few observations on its entries (see, e.g.\cite{candes2009exact, keshavan2010matrix, candes2010power, dai2010set, recht2010guaranteed, recht2011simpler, cai2010singular, lin2009fast, mazumder2010spectral}). Earlier work on matrix completion typically assume that the observations are noiseless, i.e., we may directly observe a subset of entries of $M$. In the real world, however, the observations are noisy, which is the focus of  the subsequent work \cite{keshavan2009matrix, candes2010matrix, negahban2011estimation, negahban2012restricted,rohde2011estimation,sonierror}, most of which consider a scenario where $M$ is the sum of a low-rank matrix with a Gaussian random matrix, i.e., the observations are a subset of entries of $M$ contaminated with Gaussian noise. Recently there has also been work which consider the more general noise models, including noisy 1-bit observations \cite{davenport20121}, which may be viewed as a case where the observations are Bernoulli random variables whose parameters depend on a underlying low-rank matrix. The other method \cite{soni2014noisy} is developed for Poisson matrix completion but it does not establish a lower bound. Another related work \cite{soniestimation} (not in the matrix completion setting) considers the case where {\it all} entries of the low-rank matrix are observed and the observations are Poisson counts of the entries of the underlying matrix. In the compressed sensing literature, there is a line of research for sparse signal recovery in the presence of Poisson noise \cite{raginsky2010compressed, raginsky2011performance,jiang2014minimax} and the corresponding performance bounds. The recently developed SCOPT \cite{SCOPT13, SCOPT_journal} algorithm can also be used to solve the Poisson compressed sensing problems.

In this paper, we extend the theory of matrix completion to the case of Poisson observations. We study recovery based on maximum likelihood with proper constraints on a matrix $M$ with rank less than or equal to $r$ (nuclear norm bound $\|M\|_* \leq \alpha\sqrt{r d_1 d_2}$ for some constant $\alpha$ and bounded entries $\beta \leq M_{ij}  \leq\alpha$). Note that the formulation differs from the one-bit matrix completion case in that we also require a lower bound on each entry of the matrix. This is consistent with an intuition that the value of each entry can be viewed as the signal-to-noise ratio (SNR) for a Poisson observation, and hence this essentially poses a requirement for the minimum SNR.
We also establish upper and lower bounds on the recovery error, by adapting the arguments used for one-bit matrix completion \cite{davenport20121}. The upper and lower bounds nearly match up to a factor on the order of $\mathcal{O}(\log(d_1 d_2))$, which shows that the convex relaxation formulation for Poisson matrix completion is nearly optimal. (We conjecture that such a gap is inherent to the Poisson problem).
Moreover, we also highlight a few important distinctions of Poisson matrix completion compared to the prior work on matrix completion in the absence of noise and with Gaussian noise:  (1) Although our arguments are adapted from one-bit matrix completion (where the upper and lower bounds nearly match), in the Poisson case there will be a gap between the upper and lower bounds, possibly due to the fact that Poisson distribution is only locally sub-Gaussian. In our proof, we notice that the arguments based on bounding all moments of the observations, which usually generate tight bounds for prior results with sub-Gaussian observations, do not generate tight bounds here; (2) We will need a lower bound on each matrix entry in the maximum likelihood formulation, which can be viewed as a requirement for the lowest signal-to-noise ratio (since the signal-to-noise ratio (SNR) of a Poisson observation with intensity $I$ is $\sqrt{I}$). Compared with the more general framework for $M$-estimator \cite{WainwrightReview2014}, our results are specific to the Poisson case, which may possible be stronger but do not apply generally. We also develop several simple yet efficient algorithms, including proximal and accelerated proximal gradient descent algorithms, and an algorithm which is based on singular value thresholding that we examine in details. This algorithm can be viewed as a consequence of approximating the log likelihood function by its second order Taylor expansion and invoking a theorem for exact solution of a nuclear norm regularized problem
\cite{cai2010singular}. Our algorithm is related to \cite{ji2009accelerated, wainwright2014structured, agarwal2010fast} and can be viewed as a special case where a simple closed form solution for the algorithm exists. We further demonstrate the good performance of the algorithm in recovering solar flare images.

Our formulation and results are inspired by the seminal work of one-bit matrix completion \cite{davenport20121}, yet with several important distinctions. In one-bit matrix completion, the value of each observation $Y_{ij}$ is binary-valued and hence bounded, whereas in our problem, each observation is a Poisson random variable which is unbounded; hence, the arguments involve bounding measurements have to be changed. In particular, we need to bound $\max_{ij} Y_{ij}$ when $Y_{ij}$ is a Poisson random variable with intensity $M_{ij}$. Moreover, the Poisson likelihood function is non Lipschitz (due to a bad point when $M_{ij}$ tends to zero), and hence we need to introduce a lower bound on each entry of the matrix $M_{ij}$, which can be interpreted as the lowest required SNR. Other distinctions also include analysis taking into account of the property of the Poisson likelihood function, and using Kullback-Leibler (KL) divergence as well as Hellinger distance that are different from those for the Bernoulli random variable as used in \cite{davenport20121}.

While working on this paper we realize a parallel work \cite{lafond2015low} which also studies performance bounds for low rank matrix completion with exponential family noise under more general assumptions and using a different approach for proof (Poisson noise is a special case of theirs). Their upper bound for the MSE per entry is on the order of $\mathcal{O}\left(\log(d_1 + d_2) r\max\{d_1, d_2\}/m\right)$ (our upper bound is $\mathcal{O}\left(\log(d_1 d_2)\sqrt{r(d_1+d_2)/m}\right)$), and their lower bound is on the order of $\mathcal{O}\left(r\max\{d_1, d_2\}/m\right)$ (versus our lower bound is $\mathcal{O}\left(\sqrt{r(d_1+d_2)/m}\right)$).

The rest of the paper is organized as follows. Section \ref{sec:model} sets up the formalism for Poisson matrix completion.  Section \ref{sec:method_bound} presents the matrix recovery based on constrained maximum likelihood and establishes the upper and lower bounds for the recovery accuracy. Section \ref{sec:algorithm} presents an efficient iterative algorithm that solves the maximum likelihood approximately and demonstrates its performance on recovering solar flare images. 
All proofs are delegated to Appendix.

The notation in this paper is standard. In particular, $\mathbb{R}_+$ denotes the set of positive real numbers; $[d]=\{1,2,\ldots,d\}$; $\mathbb{I}_{[\varepsilon]}$  is the indicator function for an event $\varepsilon$; $|A|$ denotes the number of elements in a set $A$; $\mbox{diag}\{\lambda_i\}$ denotes a diagonal matrix with a set of numbers $\{\lambda_i\}$ on its diagonal; $\textbf{1}_{n \times m}$ denotes an $n$-by-$m$ matrix of all ones. Let entries of a matrix $M$ be denoted by $M_{ij}$. Let $\|M\|$ be the spectral norm which is the largest absolute singular value, $\|M\|_{F} = \sqrt{\sum_{i,j}M_{ij}^2}$ be the Frobenius norm, $\|M\|_*$ be the nuclear norm which is the sum of the singular values, and finally $\|M\|_{\infty}$ = $\max_{ij}|M_{ij}|$ be the infinity norm. Let $\mbox{rank}(M)$ denote the rank of a matrix $M$.
We say that a random variable $X$ follows Poisson distribution with parameter $\lambda$ (or $X \sim \mbox{Poisson}(\lambda)$ if its probability mass function $\mathbb{P}(X=k) = e^{-\lambda}\lambda^k/(k!)$).  We also define the KL divergence and Hellinger distance for Poisson distribution as follows: the KL divergence of two Poisson distributions with parameters $p$ and $q$, where $p,q \in \mathbb{R}_+$ is given by
$
D(p\|q) \triangleq p\log(p/q) - (p-q);
$
the Hellinger distance between two Poisson distributions with parameters $p$ and $q$ with $p,q \in \mathbb{R}_+$ is given by
$
d_H^2(p, q) \triangleq 2-2\exp\left(-\frac{1}{2}\left(\sqrt{p}-\sqrt{q}\right)^2\right),
$
We further define the average KL divergence and Hellinger distance for entries of two matrices $P$, $Q \in \mathbb{R}_+^{d_1 \times d_2}$, where each entry corresponds to the parameter of a Poisson random variable:
$$
D(P\|Q) = \frac{1}{d_1 d_2}\sum_{i,j}D(P_{ij}\|Q_{ij}),
$$
$$
d_H^2(P,Q) = \frac{1}{d_1 d_2}\sum_{i,j}d_H^2(P_{ij},Q_{ij}).
$$

\section{Formulation}
\label{sec:model}

Suppose we observe a subset of entries of a matrix $M \in \mathbb{R}_+^{d_1 \times d_2}$ on the index set $\Omega \subset [d_1] \times [d_2]$. The indices are randomly selected with $\mathbb{E}|\Omega|=m$. In other words, $\mathbb{I}_{\{(i,j) \in \Omega\}}$ are i.i.d. Bernoulli random variables with parameter $m/(d_1 d_2)$.
The observations are Poisson counts of the observed matrix entries
\begin{equation}
Y_{ij} \sim  \mbox{Poisson}(M_{ij}), \quad \forall (i,j) \in \Omega.
\label{poissonmodel}
\end{equation}
Our goal is to recover the matrix $M$ from the Poisson observations $\{Y_{ij}\}_{(i,j) \in \Omega}$.

We make the following assumptions. First, we set an upper bound $\alpha>0$ for the entries of $M$ to entail the recovery problem is  well-posed \cite{negahban2012restricted}. This assumption is also reasonable in practice; for instance, $M$ may represent an image which is usually not too spiky. Second, assume the rank of $M$ is less than or equal to a positive integer $r \leq \min\{d_1,d_2\}$ (this assumption is not restrictive in that we only assume an upper bound on the rank).
The third assumption is characteristic to Poisson matrix completion: we set a lower bound $\beta>0$ for each entry $M_{ij}$. This entry-wise lower bound is required for our later analysis, and it also has an interpretation of a minimum required signal-to-noise ratio (SNR), as the SNR of a Poisson observation with intensity $I$ is $\sqrt{I}$.


We recover the matrix $M$ using a regularized maximum likelihood formulation. Note that the log-likelihood function for the Poisson observation model (\ref{poissonmodel}) is proportional to
\begin{equation}
    F_{\Omega, Y}(X) = \sum_{(i,j)\in \Omega} Y_{ij}\log X_{ij} - X_{ij},
\label{likelihood}
\end{equation}
where the subscript $\Omega$ and $Y$ indicate the random quantities involved in the maximum likelihood function $F$. Based on our assumptions, we may define a set of candidate estimators
\begin{equation}
\begin{split}
&\mathcal{S} \triangleq \left\{ X \in \mathbb{R}_+^{d_1 \times d_2} : \|X\|_* \leq \alpha \sqrt{r d_1 d_2}, \right. \\
&\qquad \qquad \qquad \qquad \left. \beta \leq X_{ij} \leq \alpha, \forall (i,j) \in [d_1] \times [d_2] \right\}.
\end{split}
\label{searchspace}
\end{equation}
Here the upper bound on the nuclear norm $\|M\|_*$ comes from combining the assumptions $\|M\|_{\infty} \leq \alpha$ and rank$(M) \leq r$, since $\|M\|_* \leq \sqrt{\mbox{rank}(M)}\|M\|_F$ and $\|M\|_F \leq \sqrt{d_1 d_2}\|M\|_{\infty}$ lead to $\|M\|_* \leq \alpha \sqrt{r d_1 d_2}$.
An estimator $\widehat{M}$ for $M$ can be obtained by solving the following convex optimization problem:
\begin{equation}
\widehat{M} = \arg \max_{X \in \mathcal{S}} F_{\Omega,Y}(X).
\label{optimization_problem}
\end{equation}

\section{Performance Bounds}
\label{sec:method_bound}

In the following, we establish an upper bound and an information theoretic lower bound on the mean square error (MSE) per entry $\|\widehat{M} - M\|_F^2/(d_1 d_2)$ for the estimator in (\ref{optimization_problem}).



\begin{theorem}[Upper bound]
\label{maintheorem}
    Assume $M \in \mathcal{S}$, $\mbox{rank}(M) = r$, $\Omega$ is chosen at random following our sampling model with $\mathbb{E}|\Omega| = m$, and $\widehat{M}$ is the solution to (\ref{optimization_problem}). Then with a probability exceeding $\left(1-C/(d_1 d_2)\right)$, we have
    \begin{equation}
    \begin{split}
       & \frac{1}{d_1 d_2} \|M-\widehat{M}\|_F^2 \leq C' \left(\frac{8\alpha T}{1-e^{-T}}\right) (\frac{\alpha \sqrt{r}}{\beta})
       \cdot \\
       &  \left( \alpha(e^2-2) + 3\log(d_1 d_2) \right) \sqrt{\frac{d_1 +d_2}{m}} \sqrt{1+\frac{(d_1+d_2)\log(d_1 d_2)}{m}}.
    \end{split}
    \label{bound:MC}
    \end{equation}
    If $m\geq (d_1+d_2)\log(d_1 d_2)$ then (\ref{bound:MC}) simplifies to
    \begin{equation}
    \begin{split}
    &\frac{1}{d_1 d_2} \|M-\widehat{M}\|_F^2 \leq \sqrt{2}C' \left(\frac{8\alpha T}{1-e^{-T}}\right)  \left(\frac{\alpha \sqrt{r}}{\beta} \right) \cdot \\
    &\left( \alpha(e^2-2) + 3\log(d_1 d_2) \right) \sqrt{\frac{d_1 + d_2}{m}}.
    \end{split}
    \label{bound:MC2}
    \end{equation}
    Above, $T, C', C$ are absolute constants.
where $T, C, C'$ are absolute constants.

\label{maintheorem}
\end{theorem}

The proof of Theorem \ref{maintheorem} is an extension of the ingenious arguments for one-bit matrix completion \cite{davenport20121}. The extension for Poisson case here is nontrivial for various aforementioned reasons (notably the non sub-Gaussian and only locally sub-Gaussian nature of the Poisson observations). An outline of our proof is as follows. First, we establish an upper bound for the KL divergence $D(M \| X)$ for any $M,X \in \mathcal{S}$ by applying Lemma \ref{firstlemma} given in the appendix. Second, we find an upper bound for the Hellinger distance $d_H^2(M, \widehat{M})$ using the fact that the KL divergence can be bounded from below by the Hellinger distance. Finally, we bound the mean squared error in Lemma \ref{secondlemma} via the Hellinger distance.
%

\begin{remark}
Fixing $d_1, d_2, m$, $\alpha$ and $\beta$, the upper bound in Theorem \ref{maintheorem} increases as $r$ increases. This is consistent with the intuition that our method is better at dealing with approximately low-rank matrices (than with nearly full rank matrices). On the other hand, fixing $d_1, d_2, \alpha$, $\beta$ and $r$,  the upper bound decreases as $m$ increases, which is also consistent with our intuition that $M$ is supposed to be recovered more accurately with more observations.
\end{remark}

\begin{remark}
In the upper bound (\ref{bound:MC}), the mean-square-error per entry can be arbitrarily small, in the sense that the upper bound goes to zero as $d_1$ and $d_2$ go to infinity when the number of the measurements $m =\mathcal{O}((d_1+d_2)\log^\delta(d_1 d_2))$ ($m\leq d_1 d_2$) for $\delta >2$ when $r$ is fixed, or for $\delta > 3$ when $r$ is sublinear on the order of $o(\log (d_1 d_2))$.
\end{remark}

%
%


The following theorem establishes an information theoretic lower bound and demonstrates that  there exists an $M \in \mathcal{S}$ such that {\it any} recovery method cannot achieve a mean square error per entry less than the order of $\mathcal{O}(r\max\{d_1, d_2\}/m)$.

\begin{theorem}[Lower bound]
Fix $\alpha$, $r$, $d_1$, and $d_2$ to be such that $\alpha, d_1, d_2 \geq 1$, $r \geq 4$, $\alpha \geq 2\beta$, and $\alpha^2 r \max\{d_1,d_2\} \geq C_0$. Let $\Omega$ be any subset of $[d_1] \times [d_2]$ with cardinality $m$. Consider any algorithm which, for any $M \in \mathcal{S}$, returns an estimator $\widehat{M}$. Then there exists $M \in \mathcal{S}$ such that with probability at least $3/4$,
\begin{equation}
\frac{1}{d_1 d_2} \|M-\widehat{M}\|_F^2 \geq \min\left\{C_1, C_2 \alpha^{3/2} \sqrt{\frac{r\max\{d_1,d_2\}}{m}}\right\}
\label{lowerbound}
\end{equation}
as long as the right-hand side of (\ref{lowerbound}) exceeds $r\alpha^2 /\min\{d_1,d_2\}$, where $C_0, C_1,C_2$ are absolute constants.
\label{maintheorem2}
\end{theorem}

Similar to \cite{davenport20121, candes2013well}, proof of Theorem \ref{maintheorem2} relies on  information theoretic arguments outlined as follows. First we find a set of matrices $\chi \in \mathcal{S}$ so that the distance between any $X^{(i)}, X^{(j)} \in \chi$, identified as $\|X^{(i)}- X^{(j)}\|_F$, is sufficiently large. Then, for any $X \in \mathcal{S}$ and the recovered $\widehat{X}$, if we assume that they are sufficiently close to each other with high probability, then we can claim that $X$ is the element in the set $\mathcal{S}$ that is closest to $\widehat{X}$. Finally, by applying a generalized Fano's inequality involving KL divergence, we claim that the probability for the event that $X$ is the matrix in set $\mathcal{S}$ closest to $\widehat{X}$ must be small, which leads to a contraction and hence proves our lower bound.

\begin{remark}
The assumptions in Theorem \ref{maintheorem2} can be achieved, for example, by the following construction. First, choose an $\alpha$ such that $\alpha \geq \max\{1, 2\beta\}$, and then an $r \geq 4$. Then, for $d_1$ (or $d_2$) sufficiently large,  the conditions that  $\alpha^2 r \max\{d_1,d_2\} \geq C_0$ and the right-hand side of (\ref{lowerbound}) exceeds $r\alpha^2 /\min\{d_1,d_2\}$ are met. Since $r \leq O(\min\{d_1,d_2\}/\alpha^2)$, $M \in \mathcal{S}$, what has been chosen is approximately low-rank. In other words, no matter how large $r$ is, we can always find $d_1$ (or $d_2$) large enough so that the assumptions in Theorem \ref{maintheorem2} are satisfied and thus there exist an $M$ which can not be recovered with arbitrarily small error by any method.
\end{remark}
\begin{remark}
When $m\geq (d_1+d_2)\log(d_1 d_2)$ and $m = \mathcal{O}(r(d_1 + d_2) \log^{\delta} (d_1 d_2))$ with $\delta>2$, the ratio between the upper bound in (\ref{bound:MC2}) and the lower bound in (\ref{lowerbound}) is on the order of $\mathcal{O}(\log(d_1 d_2))$.
Hence, the lower bound matches the upper bound up to a logarithmic factor.
\end{remark}

\section{Algorithms}
\label{sec:algorithm}

The matrix completion problem formulated in (\ref{optimization_problem}) is a Semidefinite program (SDP),  since it is a nuclear norm minimization problem with a convex feasible domain. Hence, we may solved it, for example, via the interior-point method \cite{liu2009interior}.
Although the interior-point method returns an exact solution to (\ref{optimization_problem}), it does not scale well with the dimensions of the matrix $d_1$ and $d_2$.

In the following, we will develop a set of iterative algorithms that solves the problem approximately and are more efficient than solving the problem as SDP.
In doing so, we use the framework of proximal algorithms to solve (\ref{optimization_problem}). At first, we rewrite search space $\mathcal{S}$ as the intersection of two closed and convex set in $\mathbb{R}^{d_1 \times d_2}$:
$$
\Gamma_1 \triangleq \{M \in \mathbb{R}^{d_1 \times d_2} : \|M\|_* \leq r\sqrt{d_1 d_2}\} ~\mbox{and}
$$
$$
\Gamma_2 \triangleq \{M \in \mathbb{R}^{d_1 \times d_2} : \beta \leq M_{ij}\leq \alpha, \forall (i,j) \in [d_1]\times[d_2]\},
$$
where the first set is a nuclear norm ball and the second set is a high-dimensional box. Let $f(M) \triangleq -F_{\Omega, Y}(M)$ be the negative log-likelihood function, then optimization problem (\ref{optimization_problem}) is equivalent to
\begin{equation}
\widehat{M} = \arg \min_{M \in \Gamma_1 \bigcap \Gamma_2} f(M).
\label{newoptimizationproblem}
\end{equation}
Noticing that the search space $\mathcal{S} = \Gamma_1 \bigcap \Gamma_2$ is closed and convex and $f(M)$ is a convex function, we can use proximal gradient methods to solve (\ref{newoptimizationproblem}). Let $I_{\Gamma}(M)$ be an indicator function that takes value zero if $M \in \Gamma$ and is $\infty$ if $M \in \Gamma^c$. Then problem (\ref{newoptimizationproblem}) is also equivalent to
\begin{equation}
\widehat{M} = \arg \min_{M \in \mathbb{R}^{d_1 \times d_2}} f(M) + I_{\Gamma_1 \bigcap \Gamma_2}(M).
\end{equation}

To guarantee the convergence of proximal gradient method, we need the Lipschitz constant $L>0$. In our case, Lipschitz constant $L$ is a positive number satisfying
\begin{equation}
\| \nabla f(X) - \nabla f(Y) \|_F \leq L \|X-Y\|_F, \forall X,Y \in \mathcal{S},
\label{Lipschitz}
\end{equation}
and hence $L = \alpha/\beta^2$ by the definition of our problem. Define the projection of $Y$ onto $\Gamma$ as
$$
\Pi_{\Gamma}(Y) = \arg \min_{X \in \Gamma} \| X-Y \|_F^2.
$$

\begin{algorithm}
  \caption{Proximal Gradient for Poisson Matrix Completion}
  \begin{algorithmic}[1]
    \STATE Initialize: $[M_0]_{ij} = Y_{ij}$ for $(i, j) \in \Omega$ and $[M_0]_{ij} = (\alpha+\beta)/2$ otherwise; the maximum number of iterations $K$.
    \FOR{$k = 1, 2, \ldots K $}
    \STATE $M_k = \Pi_{\mathcal{S}} (M_{k-1} - (1/L)\nabla f(M_{k-1}))$
    \ENDFOR
  \end{algorithmic}
  \label{alg:proximal1}
\end{algorithm}

Algorithm \ref{alg:proximal1} has linear convergence rate, which is established in the following theorem:
\begin{theorem}
Let $\{ M_k \}$ be the sequence generated by Algorithm \ref{alg:proximal1}. Then for any $k>1$, we have
$$
f(M_k) - f(\widehat{M}) \leq \frac{L \|M_0 - \widehat{M}\|_F^2}{2k}.
$$
\label{convergence1}
\end{theorem}
Although Algorithm \ref{alg:proximal1} can be implemented easily, its linear convergence rate is not sufficiently if the Lipschitz constant $L$ is large. In such scenarios, we prefer Nesterov's accelerated method for solving this problem which is our Algorithm \ref{alg:proximal2}.
\begin{algorithm}[h!]
  \caption{Accelerated Proximal Gradient for Poisson Matrix Completion}
  \begin{algorithmic}[1]
    \STATE Initialize: $[M_0]_{ij} = Y_{ij}$ for $(i, j) \in \Omega$ and $[M_0]_{ij} = (\alpha+\beta)/2$ otherwise; $Z_0 = M_0$; the maximum number of iterations $K$.
    \FOR{$k = 1, 2, \ldots K $}
    \STATE $M_k = \Pi_{\mathcal{S}} (Z_{k-1} - (1/L)\nabla f(Z_{k-1}))$
    \STATE $Z_k = M_k + \left( (k-1)/(k+2)\right)(M_k-M_{k-1})$
    \ENDFOR
  \end{algorithmic}
  \label{alg:proximal2}
\end{algorithm}
Algorithm \ref{alg:proximal2} has faster convergence, as stated in the following Theorem \ref{convergence2}.
\begin{theorem}
Let $\{ M_k \}$ be the sequence generated by Algorithm \ref{alg:proximal2}. Then for any $k>1$, we have
$$
f(M_k) - f(\widehat{M}) \leq \frac{2L \|M_0 - \widehat{M}\|_F^2}{(k+1)^2}.
$$
\label{convergence2}
\end{theorem}
The remaining of the problem is then to deal with the projection onto the search space $\mathcal{S}$. Since $\mathcal{S}$ is an intersection of two convex sets, we may use alternating projection algorithm to compute a sequence that converges to this intersection of $\Gamma_1$ and $\Gamma_2$, which is stated in Algorithm \ref{alg:alt_proj}.
\begin{algorithm}
  \caption{Alternating Projection Algorithm}
  \begin{algorithmic}[1]
    \STATE Initialize: $U_0$ is the matrix needed to be projected onto $\mathcal{S}$.
    \FOR{$j = 1, 2, \ldots  $}
    \STATE $V_j = \Pi_{\Gamma_1} (U_{j-1})$
    \STATE $U_j = \Pi_{\Gamma_2} (V_{j})$
    \STATE If $\| V_j - U_j \|_F \leq 10^{-6}$ then return $U_j$.
    \ENDFOR
  \end{algorithmic}
  \label{alg:alt_proj}
\end{algorithm}
Algorithm \ref{alg:alt_proj} is efficient if some closed forms of projection onto the convex sets can be achieved. Fortunately,  computation of the projection onto $\Gamma_2$ in our case is quite simple. Based on the definition of Frobenius norm, $[\Pi_{\Gamma_2}(Y)]_{ij}$ is: $\beta$ if $Y_{ij} < \beta$; $\alpha$ if $Y_{ij}>\alpha$; $Y_{ij}$ if $\beta \leq Y_{ij} \leq \alpha$. Even if there is no closed form expression for projection onto $\Gamma_1$, we can use TFOCS, a matlab package, to implement this step.

Similar to the construction in \cite{wainwright2014structured}, we may rewrite (\ref{optimization_problem}) as
\begin{equation}
\widehat{M} = \arg \min_{M \in \Gamma_2} f(M) + \lambda \|M\|_{*},
\label{originaloptimizationproblem}
\end{equation}
where $\lambda$ is a regularizing parameter that balances the goodness of data fit versus regularization.

The PMLSV algorithm can be derived as follows (in the same spirit as
\cite{ji2009accelerated}, \cite{rohde2011estimation}). Let $f(M) \triangleq -F_{\Omega, Y}(M)$ be the negative log-likelihood function. In the $k$th iteration, we may form a Taylor expansion of $f(M)$ around $M_{k-1}$, keep up to second term and then solve 
\begin{equation}
    M_k = \arg\min_{M \in \Gamma_2} \left[Q_{t_k}(M,M_{k-1}) + \lambda\|M\|_{*}\right],
\label{ouroptimizationproblem}
\end{equation}
with
\begin{align}
    Q_{t_k}(M,M_{k-1}) &\triangleq f(M_{k-1}) + \langle M-M_{k-1},\nabla f(M_{k-1}) \rangle \nonumber \\
    &~~~+ \frac{t_k}{2}\|M-M_{k-1}\|_F^{2}, \label{new}
\end{align}
where $\nabla f$ is the gradient of $f$, $t_k$  is the reciprocal of the step size in the $k$th iteration, which we will specify later.
By dropping and introducing terms independent of $M$ whenever needed (more details can be found in \cite{cao2014low}), (\ref{ouroptimizationproblem}) is equivalent to
\begin{align}
    &M_k = \nonumber\\ &\arg\min_{M } \left[\frac{1}{2} \left\| M- \left( M_{k-1} - \frac{1}{t_k}\nabla f(M_{k-1}) \right) \right\|_{F}^{2} + \frac{\lambda}{t_k}\|M\|_{*}\right].
\label{ourfinalproblem}
\end{align}
Using a theorem proved in \cite{cai2010singular}, we may show (in Appendix \ref{SVT}) that the exact solution to (\ref{ourfinalproblem}) is given by a form of Singular Value Thresholding (SVT):
\begin{equation}
    M_k = D_{\lambda/t_k} \left( M_{k-1} - \frac{1}{t_k}\nabla f(M_{k-1}) \right),
\label{Watkthiteration}
\end{equation}
where $D_{\tau}({\Sigma}) \triangleq \diag\{{(\sigma_i-\tau)_{+}}\}$ and $(x)_{+} = \max \{x,0\}$.

\begin{algorithm}
  \caption{PMLSV for Poisson Matrix Completion}
  \begin{algorithmic}[1]
    \STATE Initialize: $[M_0]_{ij} = Y_{ij}$ for $(i, j) \in \Omega$ and is $(\alpha+\beta)/2$ otherwise, the maximum number of iterations $K$, parameters $\eta$, and $L$.
    \FOR{$k = 1, 2, \ldots K $}
    \STATE $C = M_{k-1} - (1/L)\nabla f(M_{k-1})$
    \STATE $C = UD V^T$ \COMMENT{singular value decomposition}
    \STATE $D_{\rm new} = \diag((\diag(D)-\lambda/L)_{+})$
    \STATE $M_k = \Pi_{\Gamma_2} \left(U D_{\rm new}V^T\right)$
    \STATE If $f(M_{k}) > Q_L(M_k,M_{k-1})$ then $L=\eta L$, go to 4.
    \STATE If $|f(M_{k}) - Q_L(M_k,M_{k-1})| < 0.5/K$ then exit;
    \ENDFOR
  \end{algorithmic}
  \label{alg:main}
\end{algorithm}

The PMLSV algorithm is summarized in Algorithm \ref{alg:main}. In the algorithm description, $L$ is the reciprocal of the step size, $\eta > 1$ is a scale parameter to change the step size, and $K$ is the maximum number of iterations, which is  user specified: a larger $K$ leads to more accurate solution, and a small $K$ obtains the coarse solution quickly. 
If the cost function value does not decrease, the step size is shortened to change the singular values more conservatively.  The algorithm terminates when the absolute difference in the cost function values between two consecutive iterations is less than $0.5/K$.

Under the assumption that the box constraint is not binding, Algorithm \ref{alg:main} is the same as that in \cite{ji2009accelerated} and convergence analysis can be found there.

Despite of its simplicity, the PMLSV algorithm has a surprisingly good performance. With the simple initialization for $M_0$, the magnitude of the gradient is typically
small at each iteration. Hence, we can ensure $M_k$ to be belong to or be close to $\Gamma$ by choosing an appropriate step size in the $k$th iteration. 

The complexity of PMLSV is on the order of $O(d_1^2 d_2 + d_2^3)$ (which comes from the most expensive step of performing singular value decomposition). This is much lower than the complexity of solving an SDP,  which is on the order of $O(d_1^3+d_1d_2^3+d_1^2 d_2^2)$. In particular, for a $d$-by-$d$ matrix, PMLSV algorithm has complexity $\mathcal{O}(d^2)$ versus solving the SDP has complexity $\mathcal{O}(d^3)$.

\section{Numerical example}

We demonstrate the good performance of our estimator in recovering a solar flare image. The solar flare image is of size $48$-by-$48$. We break the image into 8-by-8 patches, then collect the vectorized patches into a 64-by-36 matrix: such a matrix is well approximated by a low-rank matrix, as demonstrated in Fig. \ref{fig:solar}.

\begin{figure}
\vspace{-0.2in}
\begin{center}
\includegraphics[width = 0.3\linewidth]{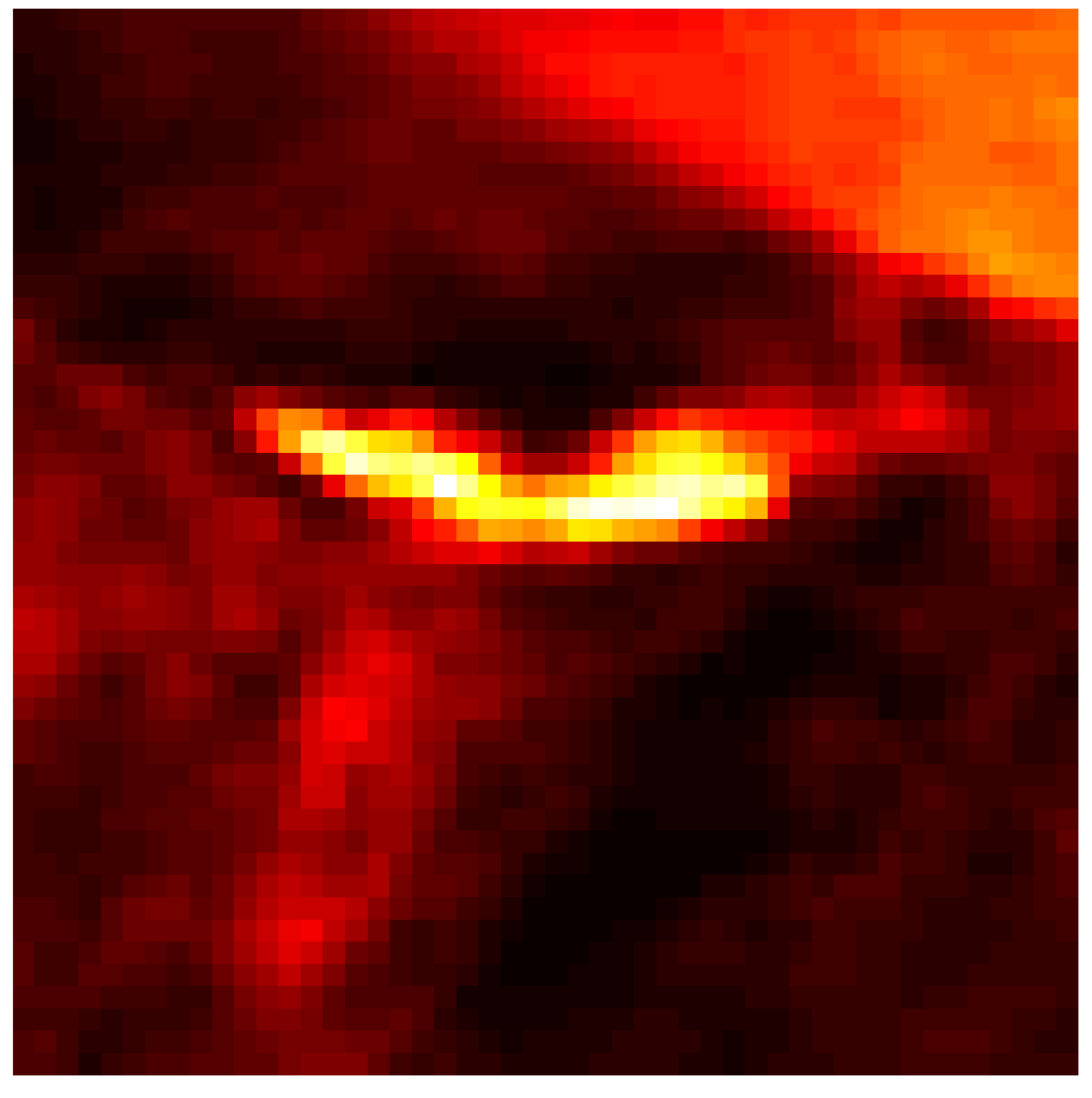}
\caption{Solar flare image of size 48-by-48 with rank 10.}
\label{fig:solar}

\end{center}\vspace{-0.25in}
\end{figure}

Suppose entries are observed using our sampling model with $\mathbb{E}|\Omega| = m$. Let $p \triangleq m/(d_1 d_2)$, then we observe $(100p)\%$ of entries. We use $L = 10^{-4}$ and $\eta=1.1$ in the PMLSV algorithm. Fig. 2 to Fig. 4 show the recovery result when $80\%$, $50\%$ and $30\%$ of the image are observed. The results show that our algorithm can recover the original image accurately when $50\%$ or above of the image  entries are observed.  In the case of only $30\%$ of the image entries are observed, our algorithm still captures the main features in the image. The PMLSV algorithm is very efficient: the running time on a laptop with 2.40Hz two core CPU and 8GB RAM for all three examples are less than $1.2$ seconds (much faster than solving SDP).

\begin{figure}[h]
\begin{center}
\begin{tabular} {cc}
\includegraphics[width = 0.3\linewidth]{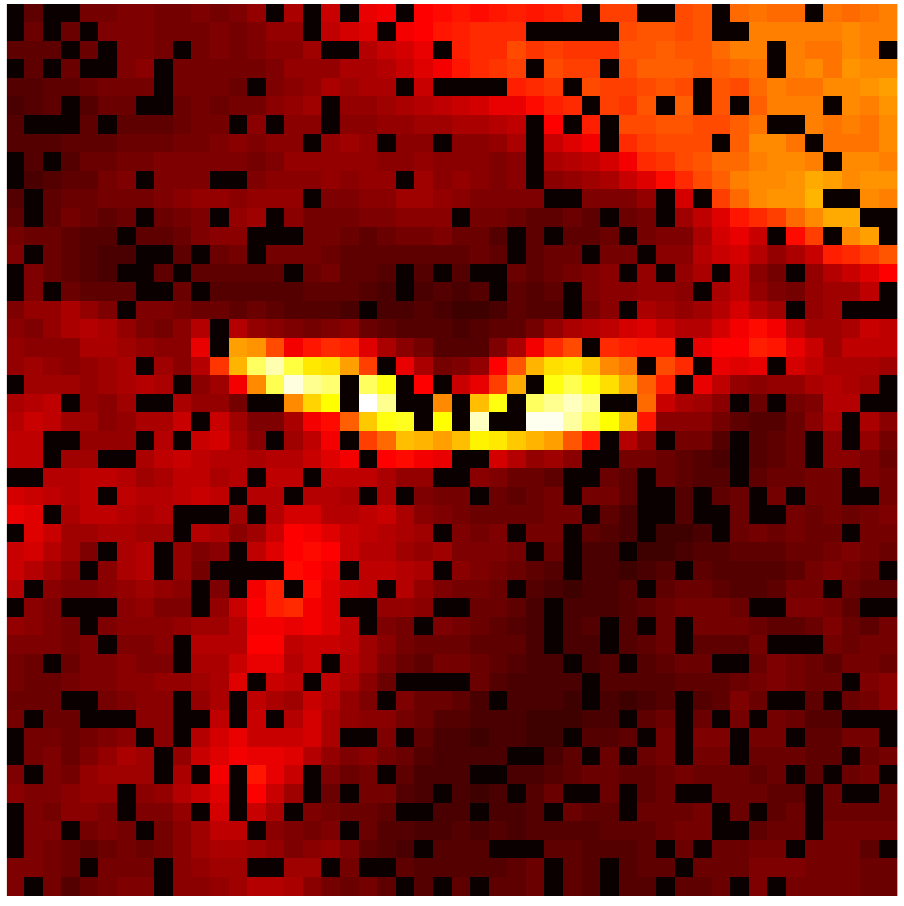} & \includegraphics[width = 0.3\linewidth]{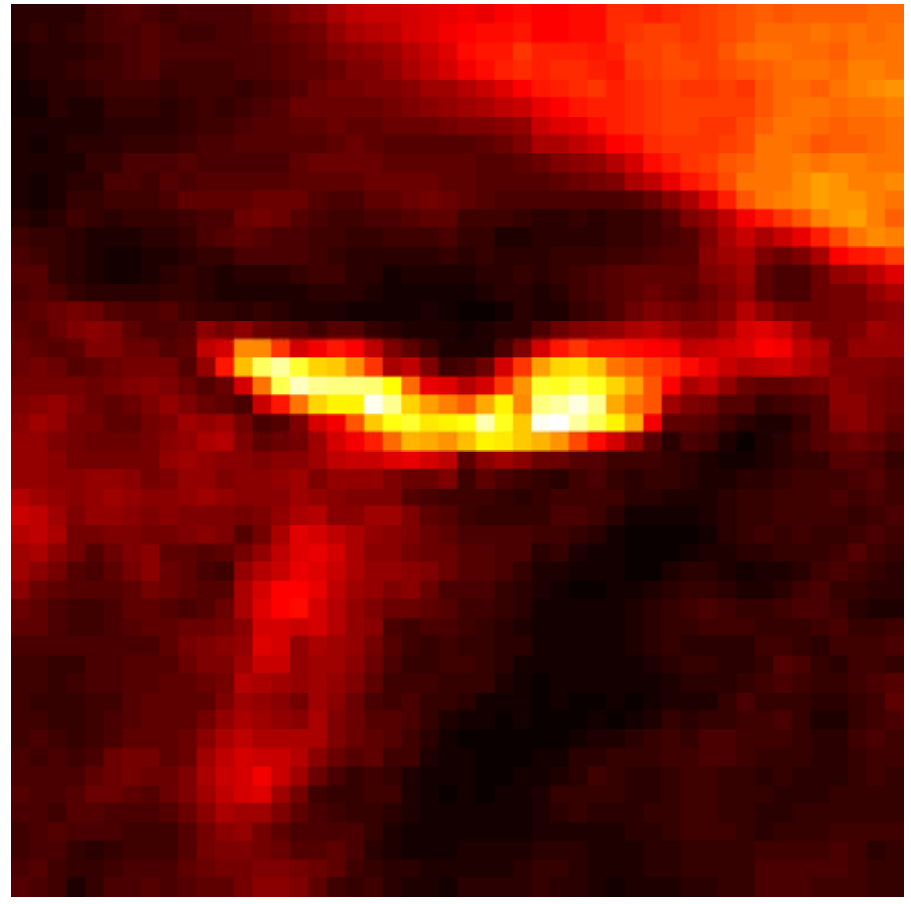} \\
(a) $p=0.8$. & (b) $\lambda=0.1, K=2000$.
\end{tabular}
\caption{(a) Observed image with $80\%$ of entries known (dark spots represent missing entries). (b) Recovered image with $\lambda=0.1$ and no more than $2000$ iterations, where the elapsed time is 1.176595 seconds.}
\end{center}
\end{figure}

\begin{figure}[h]
\begin{center}
\begin{tabular} {cc}
\includegraphics[width = 0.3\linewidth]{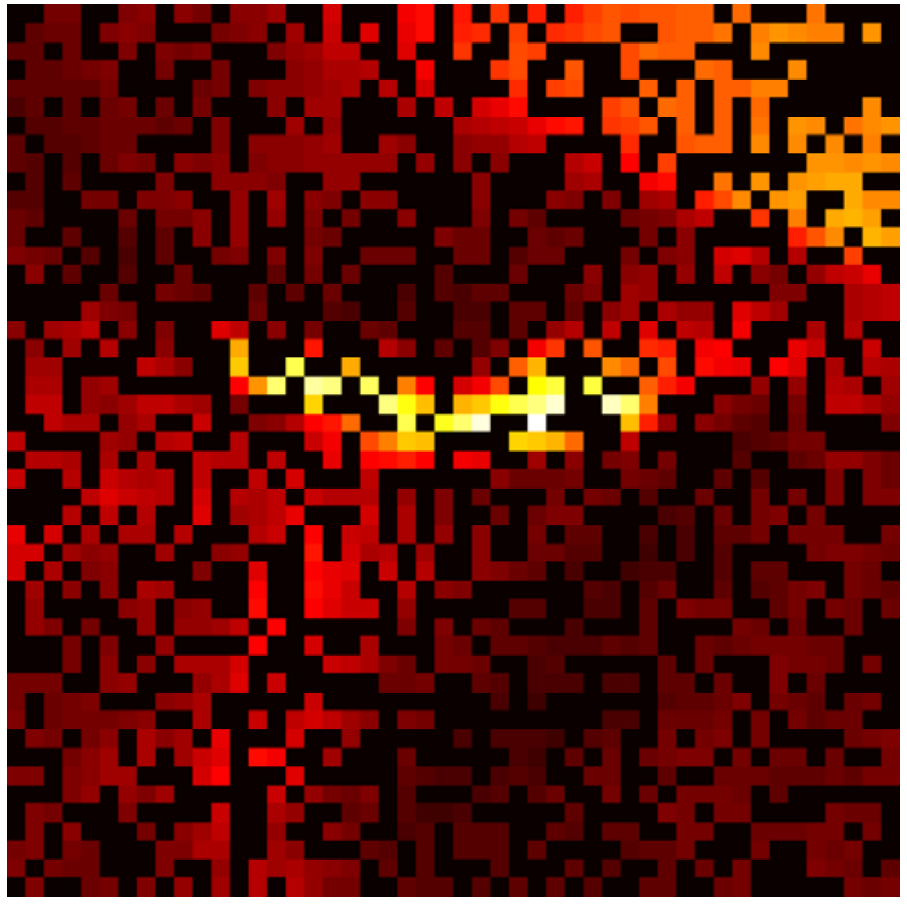} & \includegraphics[width = 0.3\linewidth]{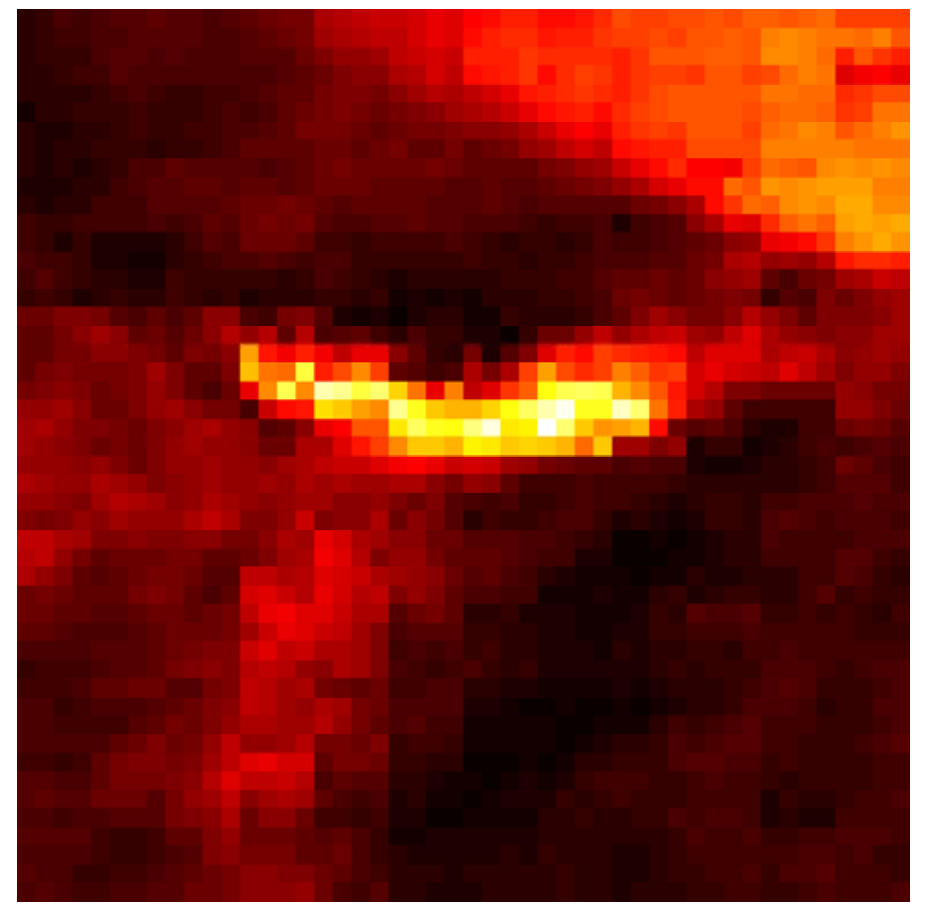} \\
(a) $p=0.5$. & (b) $\lambda=0.1, K=2000$.
\end{tabular}
\caption{(a) Observed image with $50\%$ of entries known (dark spots represent missing entries). (b) Recovered image with $\lambda=0.1$ and no more than $2000$ iterations, where the elapsed time is 1.110226 seconds.}
\end{center}
\end{figure}

\begin{figure}[h]
\begin{center}
\begin{tabular} {cc}
\includegraphics[width = 0.3\linewidth]{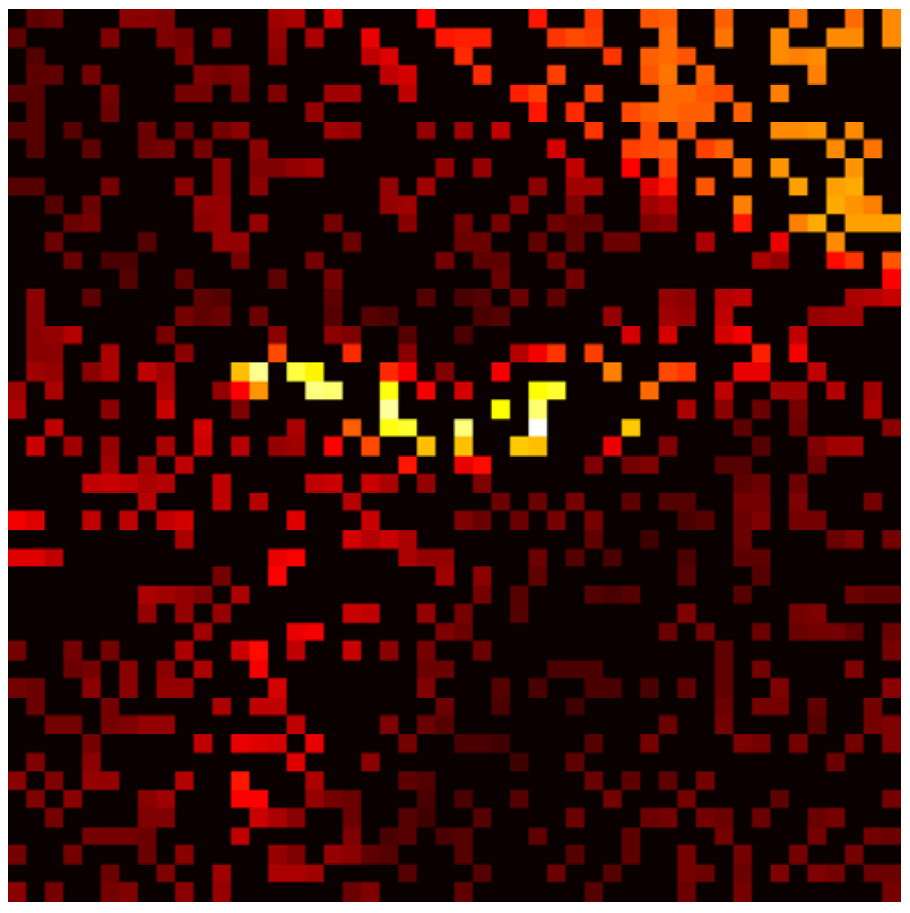} & \includegraphics[width = 0.3\linewidth]{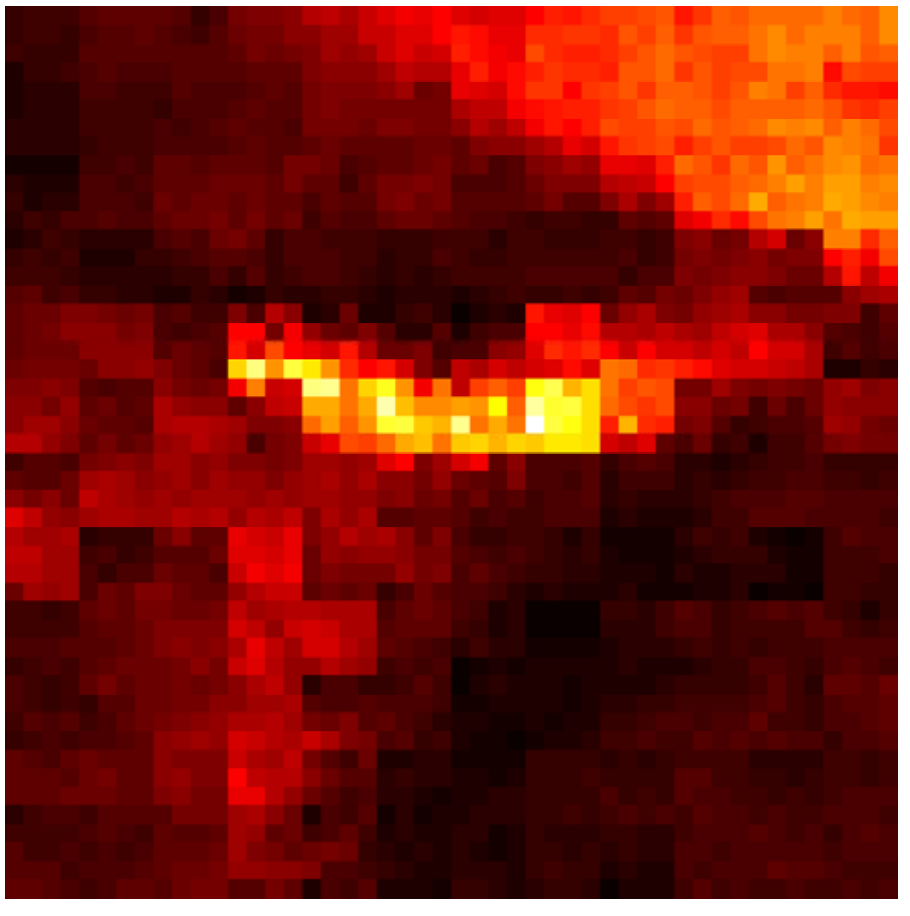} \\
(a) $p=0.3$. & (b) $\lambda=0.1, K=2000$.
\end{tabular}
\caption{(a) Observed image with $30\%$ of entries known (dark spots represent missing entries). (b) Recovered image with $\lambda=0.1$ and no more than $2000$ iterations, where the elapsed time is 1.097281 seconds.}
\end{center}
\end{figure}


\section*{acknowledgement}

The authors would like to thank Prof. Yuejie Chi, Prof. Mark Davenport, and Prof. Yaniv Plan for stimulating discussions and inspiring comments. This work is partially supported by NSF grant CCF-1442635.

\bibliography{Yang_Poisson_bound}


\appendices

\section{Singular value thresholding} \label{SVT}

Consider the following problem
\begin{equation}
\min_{Y \in \mathbb{R}^{d_1\times d_2}} \left\{ \frac{1}{2}\|Y-X\|_{F}^2+\tau\|Y\|_{*} \right\},
\label{singularvalueproblem}
\end{equation}
where $X \in \mathbb{R}^{d_1\times d_2}$ is given and $\tau$ is the regularization parameter.
For a matrix $X \in \mathbb{R}^{d_1\times d_2}$ with rank $r$, let its singular value decomposition be $X=U{\Sigma} V^T$, where $U \in \mathbb{R}^{d_1\times r}$, $V \in \mathbb{R}^{d_2\times r}$,
${\Sigma}=\diag(\{\sigma_i\},i=1,2,...,r)$, and $\sigma_i$ is a singular value of the matrix $X$.
For each $\tau\ \geq0$, define the {singular value thresholding operator} as:
\begin{equation}
    D_{\tau}(X) \triangleq U D_{\tau}({\Sigma})V^T.
\end{equation}
The solution to (\ref{singularvalueproblem}) is given by singular value thresholding according to the following theorem
\begin{theorem}[Theorem 2.1 in \cite{cai2010singular}]
For each $\tau\ \geq0$, and $X\in\mathbb{R}^{d_1\times d_2}$:
\begin{equation}
    D_{\tau}(X) = \arg \min_{Y \in \mathbb{R}^{d_1\times d_2}} \left\{ \frac{1}{2}\|Y-X\|_{F}^2+\tau\|Y\|_{*} \right\}.
\end{equation}
\label{theorem_cai}
\end{theorem}

\section{Proofs}

In the following, Lemma \ref{extendedbernsteininequality} is used in proving Lemma \ref{firstlemma}, and Lemma \ref{packingset} corresponds to Lemma 3 in \cite{davenport20121}.

\begin{lemma}
    Assuming $Y \sim$ Poisson $(\lambda)$ is a Poisson random variable with $\lambda \leq \alpha$. Then $\mathbb{P}(Y -\lambda \geq t) \leq e^{-t}, \forall t \geq t_0$ for $t_0 \triangleq \alpha (e^2-3)$.
\label{extendedbernsteininequality}
\end{lemma}

\begin{proof}

    We introduce $\theta \geq 0$,
    \begin{equation*}
    \begin{split}
        &\mathbb{P}\left( Y -\lambda \geq t\right)
        =  \mathbb{P}\left( Y \geq t + \lambda \right)\\
        =&\mathbb{P}\left(\theta Y \geq \theta\left(t+\lambda \right)\right)
        = \mathbb{P}\left(\exp\left(\theta Y \right) \geq \exp \left(\theta\left(t+ \lambda \right)\right) \right).
    \end{split}
    \end{equation*}
    Using Markov inequality, we can have
    \begin{equation*}
    \begin{split}
        &\mathbb{P}\left( Y - \lambda \geq t\right)
        \leq \exp\left(-\theta\left(t+\lambda \right)\right) \mathbb{E}\left( e^{\theta Y}\right). \\
        =& \exp(-\theta (\lambda+t)) \cdot \exp\left(\lambda (e^{\theta}-1)\right)
    \end{split}
    \end{equation*}
    Letting $\theta=2$,
    \begin{equation*}
    \begin{split}
        &\frac{\mathbb{P}\left( Y - \lambda \geq t\right)}{\exp(-t)} = \exp\left(-t + \lambda(e^2-3)\right).
    \end{split}
    \end{equation*}
    Define that
    $$
    t_0 \triangleq \alpha (e^2-3),
    $$
    to make $\frac{\mathbb{P}\left( Y-\lambda \geq t\right)}{\exp(-t)}\leq 1$, we derive that $\mathbb{P}\left( Y - \lambda \geq t\right) \leq e^{-t}$ when
    \begin{equation*}
    \begin{split}
    t\geq & t_0
    \geq \lambda (e^2-3)\\
    \end{split}
    \end{equation*}
\end{proof}

\begin{lemma}
Let $F_{\Omega, Y}(X)$ be the likelihood function defined in (\ref{likelihood}) and $\mathcal{S}$ be the set defined in (\ref{searchspace}), then
\begin{equation}
\begin{split}
&\mathbb{P} \left\{ \sup_{X\in \mathcal{S}} \left| F_{\Omega, Y}(X)-\mathbb{E}F_{\Omega, Y}(X)\right| \right.\\
& \quad \left. \geq C'\left( \alpha\sqrt{r}/\beta \right) \left( \alpha(e^2-2) + 3\log(d_1 d_2) \right) \cdot \right. \\
& \quad \left. \left(\sqrt{m(d_1+d_2)+d_1 d_2 \log(d_1 d_2)}\right) \right\} \\
& \leq \frac{C}{d_1 d_2},
\end{split}
\end{equation}
where $C'$ and $C$ are absolute positive constants and the probability and the expectation are both over $\Omega$ and $Y$.
\label{firstlemma}
\end{lemma}

\begin{proof}
In order to prove the lemma, we let
$\epsilon_{ij}$ are i.i.d. Rademacher random variables.
    In the following derivation, the first inequality is due the Radamacher symmetrization argument (Lemma 6.3 in \cite{ledoux1991probability}) and the second inequality is due to the power mean inequality: $(a+b)^h \leq 2^{h-1}(a^h+b^h)$ if $a,b>0$ and $h\geq 1$. Then we have
    \begin{equation}
        \begin{split}
        &\mathbb{E}\left[\sup_{X\in \mathcal{S}} \left| F_{\Omega, Y}(X) - \mathbb{E}F_{\Omega, Y}(X) \right|^h \right] \\
        &\leq 2^h \mathbb{E}\left[ \sup_{X\in \mathcal{S}} \left| \sum_{i,j} \epsilon_{ij} \mathbb{I}_{[(i,j)\in \Omega]} (Y_{ij}\log X_{ij} - X_{ij}) \right|^h \right]
        \end{split}\nonumber
        \end{equation}
        \begin{equation}
        \begin{split}
        &= 2^h \mathbb{E}\left[ \sup_{X\in \mathcal{S}} \left| \sum_{i,j} \epsilon_{ij} \mathbb{I}_{[(i,j)\in \Omega]} (Y_{ij}(-\log X_{ij})) \right. \right. \\
        & \qquad \qquad \quad \left. \left. + \sum_{i,j} \epsilon_{ij} \mathbb{I}_{[(i,j)\in \Omega]}X_{ij} \right|^h \right] \\
        \end{split}\nonumber
        \end{equation}
        %
        \begin{equation}
        \begin{split}
        & \leq 2^{2h-1}  \mathbb{E}\left[ \sup_{X\in \mathcal{S}} \left| \sum_{i,j} \epsilon_{ij} \mathbb{I}_{[(i,j)\in \Omega]} (Y_{ij}(-\log X_{ij})) \right|^h \right] \\
        & \quad + 2^{2h-1} \mathbb{E}\left[ \sup_{X\in \mathcal{S}} \left| \sum_{i,j} \epsilon_{ij} \mathbb{I}_{[(i,j)\in \Omega]} X_{ij} \right|^h \right],
    \end{split}
    \label{wholepart}
    \end{equation}
    where the expectation are over both $\Omega$ and $Y$.

    In the following, we will use contraction principle to further bound the first term of (\ref{wholepart}). We let $\phi(t)=-\beta \log(t+1)$. We know $\phi(0)=0$ and $|\phi^{'}(t)|=|\beta/(t+1)|$, so $|\phi^{'}(t)| \leq 1$ if $t \geq \beta-1$. Setting $Z=X-\textbf{1}_{d_1\times d_2}$, then we have $Z_{ij} \geq \beta-1, \forall (i,j) \in [d_1]\times[d_2]$ and $\|Z\|_* \leq \alpha\sqrt{rd_1 d_2} + \sqrt{d_1 d_2}$ by triangle inequality. Therefore, $\phi(Z_{ij})$ is a contraction and it vanishes at $0$. By Theorem 4.12 in \cite{ledoux1991probability} and using the fact that $|\langle A,B \rangle| \leq \|A\|\|B\|_*$, we have
    \begin{equation}
    \begin{split}
        &2^{2h-1} \mathbb{E}\left[ \sup_{X\in \mathcal{S}} \left| \sum_{i,j} \epsilon_{ij} \mathbb{I}_{[(i,j)\in \Omega]} (Y_{ij}(-\log X_{ij})) \right|^h \right] \\
        &\leq 2^{2h-1} \mathbb{E}\left[\max_{i,j} Y_{ij}^h \right]\left[ \sup_{X\in \mathcal{S}} \left| \sum_{i,j} \epsilon_{ij} \mathbb{I}_{[(i,j)\in \Omega]} ((-\log X_{ij})) \right|^h \right] \\
        &=2^{2h-1} \mathbb{E}\left[\max_{i,j} Y_{ij}^h \right] \mathbb{E}\left[ \sup_{X\in \mathcal{S}} \left| \sum_{i,j} \epsilon_{ij} \mathbb{I}_{[(i,j)\in \Omega]} \left(\frac{1}{\beta}\phi(Z_{ij})\right) \right|^h \right] \\
        &\leq 2^{2h-1}\left(\frac{2}{\beta}\right)^h \mathbb{E}\left[\max_{i,j} Y_{ij}^h \right] \mathbb{E}\left[ \sup_{X\in \mathcal{S}} \left| \sum_{i,j} \epsilon_{ij} \mathbb{I}_{[(i,j)\in \Omega]} Z_{ij}) \right|^h \right] \\
        &=2^{2h-1}\left(\frac{2}{\beta}\right)^h \mathbb{E}\left[\max_{i,j} Y_{ij}^h \right] \mathbb{E}\left[ \sup_{X\in \mathcal{S}}\left| \langle \Delta_{\Omega} \circ E , Z \rangle \right|^h \right]\\
        &\leq 2^{2h-1}\left(\frac{2}{\beta}\right)^h \mathbb{E}\left[\max_{i,j} Y_{ij}^h \right] \mathbb{E}\left[ \sup_{X \in \mathcal{S}} \|E \circ \Delta_{\Omega} \|^h \|Z\|_*^h \right] \\
        &\leq2^{2h-1}\left(\frac{2}{\beta}\right)^h \left(\alpha\sqrt{r}+1\right)^h \left(\sqrt{d_1 d_2 }\right)^h \mathbb{E}\left[\max_{i,j} Y_{ij}^h \right] \mathbb{E} \left[\|E \circ \Delta_{\Omega} \|^h\right],
    \end{split}
    \label{firstterm}
    \end{equation}
    where $E$ denotes the matrix with entries given by $\epsilon_{ij}$, $\Delta_{\Omega}$ denotes the indicator matrix for $\Omega$ and $\circ$ denotes the Hadamard product.

    Similarly, the second term of (\ref{wholepart}) can be bounded as follows:
    \begin{equation}
    \begin{split}
    &2^{2h-1} \mathbb{E}\left[ \sup_{X\in \mathcal{S}} \left| \sum_{i,j} \epsilon_{ij} \mathbb{I}_{[(i,j)\in \Omega]} X_{ij} \right|^h \right] \\
    \leq &2^{2h-1} \mathbb{E}\left[ \sup_{X \in \mathcal{S}} \|E \circ \Delta_{\Omega} \|^h \|X\|_*^h \right] \\
    \leq &2^{2h-1} \left(\alpha \sqrt{r}\right)^h \left(\sqrt{d_1 d_2}\right)^h \mathbb{E}\left[ \|E \circ \Delta_{\Omega} \|^h \right].
    \end{split}
    \label{secondterm}
    \end{equation}

    Plugging (\ref{firstterm}) and (\ref{secondterm}) into (\ref{wholepart}), we have
    \begin{equation}
    \begin{split}
    &\mathbb{E}\left[\sup_{X\in \mathcal{S}} \left| F_{\Omega, Y}(X) - \mathbb{E}F_{\Omega, Y}(X) \right|^h \right] \\
    \leq &2^{2h-1} \left(\alpha \sqrt{r}+1\right)^h \left(\sqrt{d_1 d_2}\right)^h \mathbb{E}\left[ \|E \circ \Delta_{\Omega} \|^h \right] \cdot\\
    & \left( \left(\frac{2}{\beta}\right)^h \mathbb{E}\left[\max_{i,j} Y_{ij}^h \right] +1 \right).
    \label{secondderivation}
    \end{split}
    \end{equation}

    To bound $\mathbb{E} \left[\|E \circ \Delta_{\Omega} \|^h\right]$, we can use the result from \cite{davenport20121} if we take $h=\log(d_1 d_2) \geq 1$:
    \begin{equation}
    \begin{split}
    &\mathbb{E} \left[\|E \circ \Delta_{\Omega} \|^h\right] \\
    \leq& C_0 \left(2(1+\sqrt{6})\right)^h \left( \sqrt{\frac{m(d_1+d_2)+d_1 d_2 \log(d_1 d_2)}{d_1 d_2}} \right)^h
    \end{split}\nonumber
    \end{equation}
    for some constant $C_0$. Therefore, the only term we need to bound is $\mathbb{E}\left[ \max_{i,j} Y_{ij}^{h}\right] $.

    From Lemma \ref{extendedbernsteininequality}, if $t \geq t_0$, then for any $(i,j) \in [d_1]\times [d_2]$, the following inequality holds since $t_0 > \alpha$:
    \begin{equation}
    \begin{split}
    \mathbb{P}\left( \left| Y_{ij} - M_{ij} \right| \geq t \right)
    & = \mathbb{P}\left( Y_{ij} \geq M_{ij} +  t \right) + \mathbb{P}\left( Y_{ij} \leq M_{ij} -  t \right) \\
    & \leq \exp(-t) + 0\\
    & = \mathbb{P}(W_{ij} \geq t),
    \label{exponentialappro}
    \end{split}
    \end{equation}
    where $W_{ij}$ are independent standard exponential random variables.

    Below we use the fact that for any positive random variable $X$, we can write
    $
    \mathbb{E}X = \int_{0}^{\infty}\mathbb{P}(X\geq t) dt,
    $
    allowing us to bound
    \begin{equation}
    \begin{split}
    & \mathbb{E}\left[ \max_{i,j} Y_{ij}^{h}\right] \\
    & \leq 2^{2h-1} \left( \alpha^h + \mathbb{E}\left[ \max_{i,j} \left| Y_{ij}-M_{ij} \right|^{h}\right] \right) \\
    & = 2^{2h-1} \left( \alpha^h + \int_{0}^{\infty} \mathbb{P} \left( \max_{i,j} \left| Y_{ij}-M_{ij} \right|^{h} \geq t \right) dt \right) \\
    & \leq 2^{2h-1} \left( \alpha^h + (t_0)^h + \int_{(t_0)^h}^{\infty} \mathbb{P} \left( \max_{i,j} \left| Y_{ij}-M_{ij} \right|^{h} \geq t \right) dt \right) \\
    & \leq 2^{2h-1} \left( \alpha^h + (t_0)^h + \int_{(t_0)^h}^{\infty} \mathbb{P} \left( \max_{i,j} W_{ij}^{h} \geq t \right) dt \right) \\
    & \leq 2^{2h-1} \left( \alpha^h + (t_0)^h + \mathbb{E} \left[ \max_{i,j} W_{ij}^{h}\right] \right)
    \end{split}
    \end{equation}

      Above, firstly we use triangle inequality and power mean inequality, then along with independence, we use (\ref{exponentialappro}) in the third inequality. By standard computations for exponential random variables,
    \begin{equation}
    \mathbb{E} \left[ \max_{i,j} W_{ij}^h\right] \leq 2h! + \log^{h}(d_1 d_2).
    \end{equation}
    Thus, we have
    \begin{equation}
    \begin{split}
    &\mathbb{E}\left[ \max_{i,j} Y_{ij}^{h}\right] \leq 2^{2h-1} \left( \alpha^h + (t_0)^h + 2h! + \log^{h}(d_1 d_2) \right).
    \label{onlyterm}
    \end{split}
    \end{equation}

    Therefore, combining (\ref{onlyterm}) and (\ref{secondderivation}), we have
    \begin{equation}
    \begin{split}
    &\mathbb{E}\left[\sup_{X\in \mathcal{S}} \left| F_{\Omega, Y}(X) - \mathbb{E}F_{\Omega, Y}(X) \right|^h \right] \\
    \leq &2^{4h-1} \left(\alpha \sqrt{r}+1\right)^h \left(\sqrt{d_1 d_2}\right)^h \mathbb{E}\left[ \|E \circ \Delta_{\Omega} \|^h \right] \cdot\\
    &\left(\frac{2}{\beta}\right)^h \left( \alpha^h + (t_0)^h + 2h! + \log^{h}(d_1 d_2) \right).
    \end{split}
    \end{equation}
    Then,
    \begin{equation}
    \begin{split}
    &\left(\mathbb{E}\left[\sup_{X\in \mathcal{S}} \left| F_{\Omega, Y}(X) - \mathbb{E}F_{\Omega, Y}(X) \right|^h \right]\right)^{\frac{1}{h}} \\
    \leq &16 \left(\alpha \sqrt{r}+1\right) \left(\sqrt{d_1 d_2}\right) \mathbb{E}\left[ \|E \circ \Delta_{\Omega} \|^h \right]^{\frac{1}{h}} \cdot\\
    &\left(\frac{2}{\beta}\right) \left( \alpha + t_0 + 2h + \log(d_1 d_2) \right) \\
    \leq &16 \left(\frac{2}{\beta}\right) \left(\alpha \sqrt{r}+1\right) \left(\sqrt{d_1 d_2}\right) \mathbb{E}\left[ \|E \circ \Delta_{\Omega} \|^h \right]^{\frac{1}{h}} \cdot\\
    & \left( \alpha(e^2-2) + 3\log(d_1 d_2) \right) \\
    \leq &128\left(1+\sqrt{6}\right) C_0^{\frac{1}{h}} \left(\frac{\alpha\sqrt{r}}{\beta} \right) \left( \alpha(e^2-2) + 3\log(d_1 d_2) \right) \cdot\\
    &  \left(\sqrt{m(d_1+d_2)+d_1 d_2 \log(d_1 d_2)}\right).
    \end{split}
    \end{equation}
    where we use the fact that $(a^h + b^h + c^h + d^h)^{1/h} \leq a+b+c+d$ if $a,b,c,d>0$ in the first inequality and we take $h=\log(d_1 d_2)\geq 1$ in the second inequality.

    Moreover when $C'\geq 128\left(1+\sqrt{6}\right)e$,
    $$
    C_0\left( \frac{128(1+\sqrt{6})}{C'} \right)^{\log(d_1 d_2)} \leq \frac{C_0}{d_1 d_2}
    $$

    Therefore we can use Markov inequality to see that
    \begin{equation}
    \begin{aligned}
    &\mathbb{P} \left\{ \sup_{X\in \mathcal{S}} \left| F_{\Omega, Y}(X)-\mathbb{E}F_{\Omega, Y}(X)\right| \right. \\
    & \quad \left. \geq C'\left( \alpha\sqrt{r}/\beta \right) \left( \alpha(e^2-2) + 3\log(d_1 d_2) \right) \cdot \right. \\
    & \quad \left. \left(\sqrt{m(d_1+d_2)+d_1 d_2 \log(d_1 d_2)}\right) \right\} \\
    =&\mathbb{P} \left\{ \sup_{X\in \mathcal{S}} \left| F_{\Omega, Y}(X)-\mathbb{E}F_{\Omega, Y}(X)\right|^h \right. \\
    & \quad \left. \geq C'\left( \alpha\sqrt{r}/\beta \right) \left( \alpha(e^2-2) + 3\log(d_1 d_2) \right) \cdot \right. \\
    & \quad \left. \left(\sqrt{m(d_1+d_2)+d_1 d_2 \log(d_1 d_2)}\right) \right\} \\
    \leq & \mathbb{E}\left[\sup_{X\in \mathcal{S}} \left| F_{\Omega, Y}(X) - EF_{\Omega, Y}(X) \right|^h \right]/\\
    &
    \{\left(C'\left( \alpha\sqrt{r}/\beta \right) \left( \alpha(e^2-2) + 3\log(d_1 d_2) \right) \right.\cdot \\
    &\left.\left(\sqrt{m(d_1+d_2)+d_1 d_2 \log(d_1 d_2)}\right) \right)^h \}\\
    \leq & \frac{C}{d_1  d_2}, \nonumber
    \end{aligned}
    \end{equation}
    where $C' \geq 128(1+\sqrt{6})e$ and $C$ are absolute constants.

\end{proof}

\begin{lemma}
Let $\beta \leq M_{ij}, \widehat{M}_{ij} \leq \alpha, \forall (i,j) \in [d_1] \times [d_2]$, then
$$
d_H^2(M,\widehat{M}) \geq \frac{1-e^{-T}}{4\alpha T} \frac{\|M-\widehat{M}\|_F^2}{d_1 d_2},
$$
where $T=\frac{1}{8\beta}(\alpha-\beta)^2$.
\label{secondlemma}
\end{lemma}

\begin{proof}
Assuming $x$ is any entry in $M$ and $y$ is any entry in $\widehat{M}$, then $\beta \leq x,y \leq \alpha$ and $0 \leq |x-y| \leq \alpha-\beta$.
By the mean value theorem there exists an $\xi(x,y)\in [\beta,\alpha]$ such that
$$
\hspace{-0.15in}\frac{1}{2} (\sqrt{x}-\sqrt{y})^2 = \frac{1}{2} \left(\frac{1}{2\sqrt{\xi(x,y)}}(x-y)\right)^2 = \frac{1}{8\xi(x,y)}(x-y)^2 \leq T.
$$
The function $f(z)=1-e^{-z}$ is concave in $[0,+\infty]$, so if $z\in [0,T]$, we may bound it from below with a linear function
\begin{equation}
1-e^{-z} \geq \frac{1-e^{-T}}{T}z.
\label{lemma2use}
\end{equation}
Plugging $z=\frac{1}{2} (\sqrt{x}-\sqrt{y})^2 = \frac{1}{8\xi(x,y)}(x-y)^2 $ in (\ref{lemma2use}), we have
\begin{equation}
\begin{split}
&2-2\exp\left(-\frac{1}{2} (\sqrt{x}-\sqrt{y})^2 \right) \geq \frac{1-e^{-T}}{T} \frac{1}{4\xi(x,y)}(x-y)^2 \\
&\geq \frac{1-e^{-T}}{T} \frac{1}{4\alpha}(x-y)^2.
\label{lemma2use2}
\end{split}
\end{equation}
Note that (\ref{lemma2use2}) holds for any $x$ and $y$. This concludes the proof.
\end{proof}

\begin{lemma}
Let $H \triangleq \left\{ M : \|M\|_* \leq \alpha \sqrt{r d_1 d_2}, \|M\|_{\infty} \leq \alpha \right\}$ and $\gamma \leq 1$ be such that $\frac{r}{\gamma^2}$ is an integer.
Suppose  $r/\gamma^2 \leq d_1$, then we may construct a set $\chi \in H$ of size
$$
|\chi| \geq \exp\left( \frac{r d_2}{16\gamma^2} \right)
$$
with the following properties:

1. For all $X \in \chi$, each entry has $|X_{ij}| = \alpha \gamma$.

2. For all $X^{(i)}$,$X^{(j)} \in \chi$, $i\neq j$,
$
\|X^{(i)} - X^{(j)} \|_F^2 > \alpha^2 \gamma^2 d_1 d_2/2.
$
\label{packingset}
\end{lemma}

\begin{lemma}
For $x,y >0$,
$
D(x\|y) \leq (y-x)^2/y.
$
\label{KLdivergence}
\end{lemma}
\begin{proof}
First assume $x\leq y$. Let $z=y-x$. Then $z \geq 0$ and
$
D(x\|x+z) = x\log \frac{x}{x+z} + z.
$
Taking the first derivative of this with respect to $z$, we have
$
\frac{\partial}{\partial z} D(x\|x+z) = \frac{z}{x+z}.
$
Thus, by Taylor's theorem, there is some $\xi \in [0,z]$ so that
$
D(x\|y) = D(x\|x) + z \cdot \frac{\xi}{x+\xi}.
$
Since the right-hand-side increases in $\xi$, we may replace $\xi$ with $z$ and obtain
$
D(x\|y) \leq \frac{(y-x)^2}{y}.
$
For $x>y$, with the similar argument we may conclude that for $z=y-x<0$ there is some $\xi \in [z,0]$ so that
$
D(x\|y) = D(x\|x) + z \cdot \frac{\xi}{x+\xi}.
$
Since $z<0$ and $\xi / (x+\xi)$ increases in $\xi$, then the right-hand-side is decreasing in $\xi$. We may also replace $\xi$ with $z$ and this proves the lemma.
\end{proof}

\begin{proof}[Proof of Theorem \ref{maintheorem}]
Lemma \ref{extendedbernsteininequality}, Lemma \ref{firstlemma}, and  Lemma \ref{secondlemma} are used in the proof. In the following, the expectation are taken with respect to both $\Omega$ and $\{Y_{ij}\}$. First, note that
    \begin{equation}
    F_{\Omega,Y}(X) - F_{\Omega,Y}(M) = \sum_{(i,j)\in \Omega}\left[ Y_{ij}\log\left(\frac{X_{ij}}{M_{ij}}\right) - (X_{ij}-M_{ij})\right].  \nonumber
    \end{equation}
    Then for any $X \in \mathcal{S}$,
    \begin{equation}
    \begin{aligned}
    &~~~~\mathbb{E}\left[F_{\Omega,Y}(X) - F_{\Omega,Y}(M)\right] \\
    &= \frac{m}{d_1 d_2} \sum_{i,j}\left[ M_{ij}\log\left(\frac{X_{ij}}{M_{ij}}\right) - (X_{ij}-M_{ij})\right] \\
    &=-\frac{m}{d_1 d_2} \sum_{i,j}\left[  M_{ij}\log\left(\frac{M_{ij}}{X_{ij}}\right) - (M_{ij}-X_{ij}) \right] \\
    &=-\frac{m}{d_1 d_2} \sum_{i,j} D\left(M_{ij}\|X_{ij}\right) =-m D(M\|X).
    \end{aligned}
    \end{equation}
    For $M\in \mathcal{S}$,  we know $\widehat{M} \in \mathcal{S}$ and $F_{\Omega,Y}(\widehat{M}) \geq F_{\Omega,Y}(M) $. Thus we  write
    \begin{equation*}
    \begin{aligned}
    0&\leq F_{\Omega,Y}(\widehat{M}) - F_{\Omega,Y}(M) \\
    &= F_{\Omega,Y}(\widehat{M}) + \mathbb{E}F_{\Omega,Y}(\widehat{M}) - \mathbb{E}F_{\Omega,Y}(\widehat{M}) \\
    & ~~ + \mathbb{E}F_{\Omega,Y}(M) - \mathbb{E}F_{\Omega,Y}(M) - F_{\Omega,Y}(M) \\
    &\leq \mathbb{E}\left[ F_{\Omega,Y}(\widehat{M}) - F_{\Omega,Y}(M) \right] + \\
    & ~~ \left| F_{\Omega,Y}(\widehat{M})-\mathbb{E}F_{\Omega,Y}(\widehat{M})\right| + \left|F_{\Omega,Y}(M) - \mathbb{E}F_{\Omega,Y}(M)\right| \\
    &\leq -m D(M\|\widehat{M}) + 2\sup_{X\in \mathcal{S}} \left| F_{\Omega,Y}(X) - \mathbb{E}F_{\Omega,Y}(X)\right|.
    \end{aligned}
    \end{equation*}
    Applying Lemma \ref{firstlemma}, we obtain that with probability at least $\left(1-C /(d_1 d_2)\right)$,
    \begin{equation*}
    \begin{split}
    0 &\leq -m D(M\|\widehat{M}) \\
    &+ 2C'\left( \alpha\sqrt{r}/\beta \right) \left( \alpha(e^2-2) + 3\log(d_1 d_2) \right) \cdot \\
    & \left(\sqrt{m(d_1+d_2)+d_1 d_2 \log(d_1 d_2)}\right).
    \end{split}
    \end{equation*}
    After rearranging terms and applying the fact that $\sqrt{d_1 d_2} \leq d_1 + d_2$, we obtain
    \begin{equation}
    \begin{split}
    D(M\|\widehat{M}) \leq & 2C'\left( \alpha\sqrt{r}/\beta \right) \left( \alpha(e^2-2) + 3\log(d_1 d_2) \right) \cdot \\
    & \left(\sqrt{m(d_1+d_2)+(d_1+ d_2)^2 \log(d_1 d_2)}\right).
    \label{theoremuse1}
    \end{split}
    \end{equation}
    Note that the KL divergence can be bounded below by the Hellinger distance (Chapter 3 in \cite{pollard2002user}):
    $$
    d_H^2(x,y) \leq D(x\|y).
    $$
    Thus from (\ref{theoremuse1}), we obtain
    \begin{equation}
    \begin{split}
    d_H^2(M,\widehat{M}) \leq & 2C'\left( \alpha\sqrt{r}/\beta \right) \left( \alpha(e^2-2) + 3\log(d_1 d_2) \right) \cdot \\
    & \left(\sqrt{m(d_1+d_2)+(d_1+ d_2)^2 \log(d_1 d_2)}\right).
    \end{split}
    \end{equation}
    Finally, Theorem \ref{maintheorem} is proved by applying Lemma \ref{secondlemma}.

\end{proof}

\begin{proof}[Proof of Theorem \ref{maintheorem2}]

We will prove by contradiction.
Lemma \ref{packingset} and Lemma \ref{KLdivergence} are used in the proof. Without loss of generality,  assume $d_2 \geq d_1$. Choose $\epsilon > 0$ such that
$$
\epsilon^2 = \min\left\{ \frac{1}{256}, C_2 \alpha^{3/2} \sqrt{\frac{rd_2}{m}}\right\},
$$
where $C_2$ is an absolute constant that will be be specified later.
First, choose $\gamma$ such that $\frac{r}{\gamma^2}$ is an integer and
$$
\frac{4\sqrt{2}\epsilon}{\alpha} \leq \gamma \leq \frac{8\epsilon}{\alpha} \leq \frac{1}{2}.
$$
We may make such a choice because
$$
\frac{\alpha^2 r}{64\epsilon^2} \leq \frac{r}{\gamma^2} \leq \frac{\alpha^2 r}{32\epsilon^2}
$$
and
$$
\frac{\alpha^2 r}{32\epsilon^2} - \frac{\alpha^2 r}{64\epsilon^2}  = \frac{\alpha^2 r}{64\epsilon^2} > 4\alpha^2 r > 1.
$$
Furthermore, since we have assumed that $\epsilon^2$ is larger than $C r\alpha^2/d_1$, $r/\gamma^2 \leq d_1$ for an appropriate choice of $C$.
Let $\chi'_{\alpha/2, \gamma}$ be the set defined in Lemma \ref{packingset}, by replacing $\alpha$ with $\alpha/2$ and with this choice of $\gamma$. Then we can construct a packing set
$\chi$ of the same size as $\chi'_{\alpha/2, \gamma}$ by defining
$$
\chi \triangleq \left\{ X' + \alpha\left(1-\frac{\gamma}{2}\right) \textbf{1}_{d_1 \times d_2} : X' \in \chi'_{\alpha/2, \gamma} \right\}.
$$
The distance between pairs of elements in $\chi$ is bounded since
\begin{equation}
\|X^{(i)} - X^{(j)} \|_F^2 \geq \frac{\alpha^2}{4}\frac{\gamma^2 d_1 d_2}{2} \geq 4 d_1 d_2 \epsilon^2.
\end{equation}
Define $\alpha' \triangleq (1-\gamma)\alpha$, then every entry of $X \in \chi$ has $X_{ij} \in \{\alpha, \alpha'\}$. Since we have assumed $r \geq 4$, for every $X \in \chi$, we have
\begin{align*}
\|X\|_* &= \| X' + \alpha\left(1-\frac{\gamma}{2}\right) \textbf{1}_{d_1 \times d_2}\|_* \leq \|X'\|_* + \alpha(1-\frac{\gamma}{2})\sqrt{d_1 d_2}\\
&\leq \frac{\alpha}{2} \sqrt{r d_1 d_2} + \alpha \sqrt{d_1 d_2}
\leq \alpha \sqrt{r d_1 d_2},
\end{align*}
for some $X' \in \chi'_{\alpha/2, \gamma}$.
Since the $\gamma$ we choose is less than $1/2$, $\alpha'$ is greater than $\alpha/2$. Therefore, from the assumption that $\beta \leq \alpha/2$, we conclude that $\chi \subset \mathcal{S}$.

Now consider an algorithm that for any $X \in \mathcal{S}$ returns $\widehat{X}$ such that
\begin{equation}
\frac{1}{d_1 d_2} \|X-\widehat{X}\|_F^2 < \epsilon^2
\label{assumption}
\end{equation}
with probability at least $1/4$. Next, we will show this leas to an contradiction. Let
$$
X^* = \arg \min_{X^{(i)} \in \chi} \|X^{(i)} - \widehat{X}\|_F^2,
$$
by the same argument as that in \cite{davenport20121}, we have $X^* = X$ as long as (\ref{assumption}) holds. Using the assumption that (\ref{assumption}) holds with probability at least
$1/4$, we have
\begin{equation}
\mathbb{P} (X^* \neq X) \leq \frac{3}{4}.
\label{inequality1}
\end{equation}
Using a generalized Fano's inequality for the KL divergence in \cite{yu1997assouad}, we have
\begin{equation}
\mathbb{P} (X^* \neq X) \geq 1- \frac{\max_{X^{(k)} \neq X^{(l)}} D(X^{(k)} \| X^{(l)})+1}{\log |\chi|}.
\label{inequality2}
\end{equation}
Define
$
D \triangleq D(X^{(k)}\|X^{(l)}) = \sum_{(i,j) \in \Omega} D(X^{(k)}_{ij}\| X^{(l)}_{ij} ).
$
We know that each term in the sum is either $0$, $D(\alpha\|\alpha')$, or $D(\alpha'\|\alpha)$. From Lemma \ref{KLdivergence}, since $\alpha' < \alpha$, we have
$$
D \leq \frac{m(\gamma \alpha)^2}{\alpha'} \leq \frac{64 m\epsilon^2}{\alpha'}.
$$
Combining (\ref{inequality1}) and (\ref{inequality2}), we have that
\begin{equation}
\begin{split}
\frac{1}{4} &\leq 1-\mathbb{P}(X \neq X^*) \leq \frac{D+1}{\log |\chi|} \\
&\leq 16\gamma^2 \left(\frac{\frac{64 m\epsilon^2}{\alpha'}+1}{rd_2} \right) \leq 1024\epsilon^2\left(\frac{\frac{64 m\epsilon^2}{\alpha'}+1}{\alpha^2 rd_2} \right).
\end{split}
\label{contradiction}
\end{equation}
Suppose $64m\epsilon^2 \leq \alpha'$, then with (\ref{contradiction}), we have
$$
\frac{1}{4} \leq 1024 \epsilon^2 \frac{2}{\alpha^2 r d_2},
$$
which implies that $\alpha^2 r d_2 \leq 32$. Then if we set $C_0>32$, this leads to a contradiction.
Next, suppose $64m\epsilon^2 > \alpha'$, then with (\ref{contradiction}), we have
$$
\frac{1}{4} < 1024\epsilon^2 \left( \frac{128m\epsilon^2}{(1-\gamma)\alpha^3 rd_2} \right).
$$
Since $1-\gamma > 1/2$, we have
$$
\epsilon^2 > \frac{\alpha^{3/2}}{1024}\sqrt{\frac{rd_2}{m}}.
$$
Setting $C_2 \leq 1/4096$, this leads to a contradiction. Therefore, (\ref{assumption}) must be incorrect with probability at least $3/4$. This concludes our proof.

\end{proof}

\begin{lemma}
If $f$ is a closed convex function satisfying Lipschitz condition (\ref{Lipschitz}), then for any $X,Y \in \mathcal{S}$, the following inequality holds:
$$
f(Y) \leq f(X) + \langle \nabla f(X), Y-X \rangle + \frac{L}{2} \|Y-X\|_F^2.
$$
\label{UPlip}
\end{lemma}

\begin{proof}
Let $Z$ = $Y-X$, then we have
\begin{equation}
\begin{split}
&f(Y) = f(X) +  \langle \nabla f(X), Z \rangle + \int_0^1 \langle \nabla f(X+tV)- \nabla f(X), Z \rangle ~dt \\
\leq & f(X) +  \langle \nabla f(X), Z \rangle + \int_0^1 \|f(X+tV)- \nabla f(X)\|_F\|Z\|_F~ dt \\
\leq & f(X) +  \langle \nabla f(X), Z \rangle + \int_0^1 Lt \|Z\|_F^2~ dt \\
= &  f(X) + \langle \nabla f(X), Y-X \rangle + \frac{L}{2} \|Y-X\|_F^2,
\end{split} \nonumber
\end{equation}
where we use Taylor expansion with integral remainder in the first line, the fact that dual norm of Frobenius norm is itself in the second line and Lipschitz condition
in the third line.

\end{proof}

\begin{proof}[Proof of Theorem \ref{convergence1}]
As is well known, proximal mapping of a $Y \in \mathcal{S}$ associated with a closed convex function $h$ is given by
$$
\mbox{prox}_{th}(Y) \triangleq \arg \min_{X} \left( t \cdot h(X) + \frac{1}{2}\| X-Y \|_F^2 \right),
$$
where $t>0$ is a multiplier. Define for each $M \in \mathcal{S}$ that
$$
G_t(M) \triangleq \frac{1}{t} \left( M - \mbox{prox}_{th} \left(M-t\nabla f(M)\right) \right),
$$
then we can know by the characterization of subgradient that
\begin{equation}
G_t(M) - \nabla f(M) \in \partial h(M),
\label{subgradient}
\end{equation}
where $\partial h(M)$ is the subdifferential of $h$ at $M$.
Noticing that $M-G_t(M) \in \mathcal{S}$, then from Lemma \ref{UPlip} we have
\begin{equation}
f(M-t G_t(M)) \leq f(M) - \langle \nabla f(M), t G_t(M) \rangle+ \frac{t}{2} \|G_t(M)\|_F^2,
\label{UPlip2}
\end{equation}
for all $0\leq t \leq 1/L$.
In our case, $h(M) = I_{\mathcal{S}}(M)$. Defining $g(M) \triangleq f(M) + h(M)$, combining
(\ref{subgradient}) and (\ref{UPlip2}) and using the fact that $f$ and $h$ are convex functions, we have for any $Z\in \mathcal{S}$
\begin{equation}
g(M-tG_t(M)) \leq g(Z) + \langle G_t(M), M-Z \rangle - \frac{t}{2} \|G_t(M)\|_F^2.
\label{eq1}
\end{equation}
Taking $Z=\widehat{M}$ in (\ref{eq1}), then we have for any $k \geq 0$
\begin{equation}
\begin{split}
g(M_{k+1}) - g(\widehat{M}) \leq & \langle G_t(M_{k}), M_k - \widehat{M} \rangle - \frac{t}{2} \|G_t(M_k)\|_F^2 \\
= & \frac{1}{2t} \left( \|M_k - \widehat{M}\|_F^2 - \|M_{k+1}-\widehat{M}\|_F^2 \right),
\end{split}
\label{oneiteration}
\end{equation}
where we use the fact that $\langle M,M \rangle = \|M\|_F^2$.
By taking $Z = M_k$ in (\ref{eq1}) we know that $f(M_{k+1}) < f(M_k)$ for any $k \geq 0$, so we have by also taking $t=1/L$
\begin{equation}
\begin{split}
&g(M_k) - g(\widehat{M}) \\
\leq & \frac{1}{k} \sum_{i=0}^{k-1} \left( g(M_{i+1}) - g(\widehat{M}) \right) \\
\leq &  \frac{L}{2k} \sum_{i=0}^{k-1} \left(\|M_i - \widehat{M}\|_F^2 - \|M_{i+1}-\widehat{M}\|_F^2  \right) \\
\leq &  \frac{L \|M_0 - \widehat{M}\|_F^2}{2k}.
\end{split}
\end{equation}
Finally, we proves the theorem by noticing that $h(M_k) = h(\widehat{M})=0$ for any $k \geq 0$.

\end{proof}

\begin{proof}[Proof of Theorem \ref{convergence2}]
We will use some results in the above proof.
Defining that $V_0 = M_0$ and for any $k\geq 1$,
$$
a_k \triangleq \frac{2}{k+1}, ~V_k \triangleq M_{k-1} + \frac{1}{a_k}\left(M_{k} - M_{k-1}\right).
$$
Setting $t=1/L$, then by noticing that
$$
M_k = Z_{k-1} - tG_t(Z_{k-1}),
$$
we can rewrite $V_k$ as
$$
V_k = V_{k-1} - \frac{t}{a_k}G_t(Z_{k-1}).
$$
Taking $Z=M_{k-1}$ and $Z=\widehat{M}$ in (\ref{eq1}) and make convex combination we have
\begin{equation}
\begin{split}
g(M_k) \leq & (1-a_k)g(M_{k-1})+a_k g(\widehat{M}) \\
            &+a_k \langle G_t(Z_{k-1}), V_{k-1}-\widehat{M} \rangle - \frac{t}{2} \|G_t(Z_{k-1})\|_F^2 \\
=& (1-a_k)g(M_{k-1})+a_k g(\widehat{M}) \\
&+ \frac{a_k^2}{2t} \left( \|V_{k-1} - \widehat{M}\|_F^2 - \|V_{k}-\widehat{M}\|_F^2 \right).
\end{split}
\end{equation}
Rearranging the terms before we have
\begin{equation}
\begin{split}
&\frac{1}{a_k^2}(g(M_{k})-g(\widehat{M}))+\frac{1}{2t}\|V_k-\widehat{M}\|_F^2 \leq \\
&\frac{1-a_k}{a_k^2}(g(M_{k-1})-g(\widehat{M}))+\frac{1}{2t}\|V_{k-1}-\widehat{M}\|_F^2.
\end{split}
\label{leq1}
\end{equation}
Noticing that $(1-a_k)/(a_k^2) \leq 1/(a_{k-1}^2)$ for any $k\geq 1$, we apply inequality (\ref{leq1}) recursively to get
\begin{equation}
\frac{1}{a_k^2}(g(M_{k})-g(\widehat{M}))+\frac{1}{2t}\|V_k-\widehat{M}\|_F^2 \leq \frac{1}{2t} \|M_0 - \widehat{M} \|_F^2,
\end{equation}
which proves the theorem.

\end{proof}

\end{document}